\documentclass{llncs}
\usepackage{llncsdoc}

\usepackage{graphicx,url}
\usepackage{float}

\usepackage{graphicx,url}
\usepackage{amsfonts,amsmath,setspace}
\usepackage[usenames,dvipsnames]{color}
\usepackage{wasysym}
\usepackage{float}
\usepackage{tabularx}

\newcommand{\lab}{{\cal L}}
\newcommand{\fnn}{{\sf NNF}}

\newcommand{\mfun}{{\sf m}}
\newcommand{\eqm}{=_{\mfun}}

\newcommand{\eql}{\approx}
\newcommand{\neql}{\not \approx}

\newcommand{\judgement}{{\cal S}}

\newcommand{\dom}{{\sf dom}}

\newcommand{\intcan}{\mathcal{I}_{\sf c}}
\newcommand{\Deltacan}{\Delta^{\intcan}}

\renewcommand{\int}{\mathcal{I}}
\newcommand{\intext}{{\mathcal{I}}_{\mfun}}
\newcommand{\Deltaext}{\Delta^{\intext}}

\newcommand{\unfold}{{\sf set}}

\newcommand{\aboxL}{\mathcal{A}_{\lab}}

\newcommand{\aboxLi}[1]{\mathcal{A}_{#1}}

\newcommand{\eqL}{\eql}

\newcommand{\boundM}{{\sf concepts}(\mathcal{M})}

\newcommand{\equate}[3]{{\sf equate}(#1,#2,#3)}
\newcommand{\makedifferent}[3]{{\sf differenciate}(#1,#2,#3)}

\begin{document}
\begin{frontmatter}              

\title{
Reasoning for ${\mathcal{ALCQ}}$ extended with a flexible  meta-modelling
hierarchy  
}

\author{ Regina Motz \inst{2} \and Edelweis Rohrer\inst{2} \and Paula Severi\inst{1} }

\authorrunning{Motz, Rohrer, Severi}

\institute{Department of Computer Science, University of Leicester, England \\
\email{ps330@leicester.ac.uk} \\
\and
Instituto de Computaci\'on, Facultad de Ingenier\'ia, \\  Universidad de la Rep\'ublica, Uruguay\\
\email{\{rmotz, erohrer\}@fing.edu.uy} 
}

\maketitle

\begin{abstract}
This works is motivated by a real-world case study where it is necessary to
integrate and relate  existing ontologies  through {\em meta-modelling}. 
For this, we introduce the Description Logic $\mathcal{ALCQM}$
which is obtained from   $\mathcal{ALCQ}$  
by adding  statements  that equate  individuals to
concepts in a knowledge base. 
In this new extension, a concept can be an individual of  
another concept  (called {\em meta-concept})
which themselves can be individuals of  yet another concept 
(called {\em meta meta-concept}) and so on. We  define a tableau algorithm for checking  consistency of an ontology in  $\mathcal{ALCQM}$ and prove its correctness.

\keywords{Description Logic, Meta-modelling, Meta-concepts, Well founded sets, 
 Consistency, Decidability }
\end{abstract}

\end{frontmatter}

\section{Introduction}
\label{sec:Introduction}

Our extension of $\mathcal{ALCQ}$ is motivated by a real-world application
on geographic objects  that 
requires to reuse existing ontologies and relate them through meta-modelling
 \cite{bworld}. \\
Figure 1 describes a simplified scenario of this application  
in order to illustrate the meta-modelling relationship. 
It shows two ontologies  separated by a line.
The two ontologies conceptualize the same entities at different levels of 
granularity. 
In the  ontology above the line, rivers and lakes are formalized as  individuals 
while in the one below the line they are concepts.
If we want to integrate these ontologies into a single ontology
(or into an ontology network) 
it is necessary to interpret 
the individual $river$ and the concept $River$ as the same real object.
Similarly for $lake$ and $Lake$.\\
Our solution consists in equating  the individual $river$  to the concept
  $River$ and the individual  $lake$ to the concept $Lake$. 
These equalities are called  \emph{meta-modelling axioms} and in this case, we
say that the ontologies are related through {\em meta-modelling}.
In  Figure \ref{fig:firstView}, 
 meta-modelling axioms are represented by  dashed edges.
After adding the meta-modelling axioms for rivers and lakes, 
the concept  $HydrographicObject$  is now also a {\em meta-concept}
because it is a concept that contains an individual which is also a concept. \\
The kind of meta-modelling we consider in this paper
 can be expressed in OWL Full but it cannot be
expressed in OWL DL.
The fact that it is expressed in OWL Full is not very useful since 
the meta-modelling provided by OWL Full is so expressive that leads to 
undecidability \cite{DBLP:conf/semweb/Motik05}.\\
OWL 2 DL has a very restricted form of meta-modelling 
called {\em punning} where the same identifier can be used as
an individual and as a concept \cite{FOST}. These identifiers are treated
as different objects by the reasoner and
 it is not possible  to detect certain inconsistencies.
 We next illustrate two examples where OWL would not detect inconsistencies
 because the identifiers, though they look syntactically equal,
  are actually  different. 
\begin{example}
 \label{example:five}
If we introduce 
 an axiom expressing that  \emph{HydrographicObject} is a subclass of \emph{River}, 
then  OWL's reasoner will not detect that 
the interpretation of $River$ is not a well founded set
(it is a set that  belongs to itself).
\end{example}
\begin{example}
\label{example:simpletransference}
We add  two axioms, the first one says that
 $river$ and $lake $ as individuals are equal
and the second one says that the classes $River$ and $Lake$ are
disjoint.
 Then 
OWL's reasoner  does not detect that there is a contradiction.
\end{example}
\begin{figure}[H]
\centering
\includegraphics[width=0.5\linewidth]{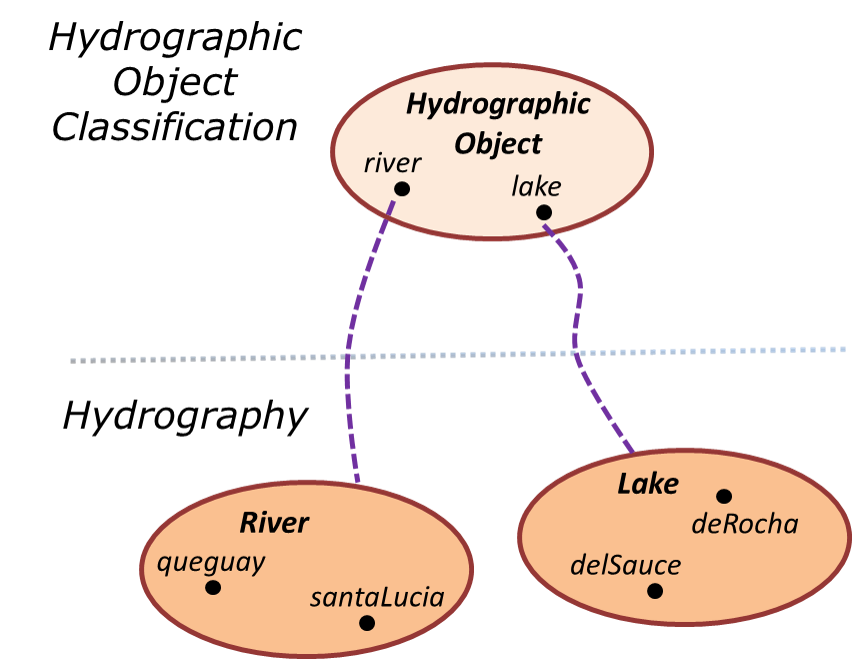}
\caption{Two ontologies  on Hydrography  }
\label{fig:firstView}
\end{figure}
In this paper, we consider  $\mathcal{ALCQ}$ ($\mathcal{ALC}$ with
qualified cardinality restrictions) and extend it
with {\em Mboxes}.  
An Mbox is a set of  equalities of the form
$a \eqm A$ where $a$ is an individual and $A$ is a concept.
In our example, we have that $river \eqm River$ and these two identifiers
are semantically equal, i.e.
the interpretations of the individual $river$ and the concept $River$ are the same.
The domain of an interpretation cannot longer consists of  only  basic objects 
but it must be any well-founded set. 
The well-foundness of our model is not ensured by means of fixing layers 
beforehand as in \cite{DBLP:conf/owled/PanHS05,DBLP:journals/ijsi/JekjantukGP10}
 but it is our reasoner which  checks for circularities.
Our approach allows the user to have any number of levels (or layers) 
(meta-concepts, meta meta-concepts and so on). 
The user does not have to write or know the layer of the concept 
because the reasoner will infer it for him.
In this way, axioms can also naturally mix elements of different layers  
and the user has the flexibility of changing the status of an individual at 
any point without having to make any substantial change to the ontology. \\
We define a tableau  algorithm for checking  consistency of
an ontology in  $\mathcal{ALCQM}$ by adding new rules and a new condition  to
 the tableau algorithm for $\mathcal{ALCQ}$. 
 The new rules  deal with the equalities and inequalities between individuals
 with meta-modelling  which need to be transferred to the level of concepts
 as equalities and inequalities between the corresponding concepts.
 The new condition  deals  with circularities avoiding non well-founded sets.
From the practical point of view, 
extending  tableau for $\mathcal{ALCQ}$
 has the advantage that one can easily  change and reuse the 
code of existing OWL's reasoners.
From the theoretical point of view, we give an elegant proof of
 correctness  by showing an isomorphism between the canonical interpretations
 of $\mathcal{ALCQ}$ and $\mathcal{ALCQM}$.
Instead of re-doing inductive proofs, we ``reuse'' and invoke  the
results of Correctness of the tableau Algorithm for $\mathcal{ALCQ}$ 
from \cite{journals/ai/BaaderBH96} wherever possible.

\paragraph{Related Work.}
\label{sec:relWork}

As we mentioned before, 
OWL 2 DL has a very restricted form of meta-modelling called {\em punning} \cite{FOST}.
In spite of the fact that the same identifier can be used  simultaneously 
as an individual and as a concept, they are  semantically different.
In order to use the punning of OWL 2 DL in the
 example of Figure \ref{fig:firstView},   
 we could change the name $river$ to $River$ and $lake$ to $Lake$.
In spite of the fact that the identifiers look syntactically equal,
OWL would not detect certain inconsistencies as the ones  illustrated
in  Examples \ref{example:five} and \ref{example:simpletransference}, and in  Example 
\ref{example:transferenceequality} which appears in Section \ref{sec:ALCQM}.   
In the first example, OWL won't detect that there is a circularity and
in the other examples, OWL won't detect that there is a contradiction. 
Apart from having the  disadvantage of not detecting certain inconsistencies,
this approach is not natural for reusing  ontologies.
For these scenarios, it is more useful to assume 
the identifiers be syntactically different  and  allow the user
to equate them by using axioms of the form $a \eqm A$. \\ 
Motik proposes a  solution for meta-modelling that is not so expressive
as RDF but which is decidable \cite{DBLP:conf/semweb/Motik05}. 
Since his syntax does not restrict the
 sets of individuals, concepts and roles to be pairwise disjoint,
 an identifier can be used as a concept and  an  individual at the same time.
From the point of view of ontology design,
we consider more natural to assume that the
 identifiers for a concept and individual that conceptualize the same real 
 object (with different granularity) 
will be  syntactically different (because most likely 
they will live  in different ontologies).
In \cite{DBLP:conf/semweb/Motik05}, Motik also  defines  two alternative semantics:
 the context approach and the HiLog approach \cite{DBLP:conf/semweb/Motik05}. 
The context approach is similar to  the so-called punning supported by OWL 2 DL.
The HiLog semantics looks more useful than the context semantics since it can detect
the inconsistency of Example \ref{example:simpletransference}.
However, this semantics ignores the issue on well-founded sets.
Besides, this semantics does not look 
 either intuitive or direct as ours since  
it uses some intermediate extra functions to interpret individuals with meta-modelling.
The algorithm  given in \cite[Theorem 2]{DBLP:conf/semweb/Motik05}
does not check for circularities  (see Example \ref{example:five})
which is one of the main contributions of this paper.
\\ 
De Giacomo et al. specifies a new formalism, ``Higher/Order Description Logics'', 
that allows to treat the same symbol of the signature as an instance, a concept and
 a role \cite{DBLP:conf/aaai/GiacomoLR11}. 
 This approach is similar to punning in the sense that the three new symbols are treated as independent elements. \\ 
Pan et al address meta-modelling by defining different ``layers'' or ``strata'' within a knowledge base \cite{DBLP:conf/owled/PanHS05,DBLP:journals/ijsi/JekjantukGP10}. This approach 
 forces  the user to explicitly write the information of the layer in the concept. This has several disadvantages:  the user should know  beforehand  in which layer the concept lies and it does not  give the flexibility of changing the layer in which it lies. Neither it  allows us to mix different layers when building concepts, inclusions or roles, e.g. we cannot express that the intersection of concepts in two different layers is empty or define a role whose  domain and range live in  different layers.
\\ 
Glimm et al. codify meta-modelling within OWL DL  \cite{DBLP:conf/semweb/GlimmRV10}. 
This codification consists in adding some extra individuals, axioms and roles to the original ontology  in order to represent meta-modelling of concepts. As any codification, this approach has the disadvantage of being involved and difficult to use,  since adding new concepts implies adding a lot of extra axioms.  
This codification is not enough for detecting inconsistencies coming from meta-modelling
(see Example \ref{example:transferenceequality}).
 The approach in  \cite{DBLP:conf/semweb/GlimmRV10} has also other limitations from the point of view of expressibility, e.g. it has only  two levels of meta-modelling (concepts and meta-concepts).

\paragraph{Organization of the paper.}

The remainder of this paper is organized as follows. 
Section \ref{section:casestudy} shows a case study and
explains the advantages of our approach.
Section \ref{sec:ALCQM} defines
 the syntax and semantics of ${\mathcal{ALCQM}}$. 
Section \ref{sec:ReasonerALCQM}  proposes an algorithm for checking consistency. 
 Section \ref{sec:CorrectTableau} proves  its correctness. 
  Finally, Section \ref{sec:relWorks} sets the future work.

\section{Case Study on Geography}
\label{section:casestudy}

In this section, we illustrate some important  advantages of our approach through  
the  real-world example on geographic objects presented in the introduction. \\ 
Figure \ref{fig:secondView} extends the ontology network given in Figure \ref{fig:firstView}.
Ontologies are delimited by light dotted lines. 
Concepts are denoted by ovals and individuals by small filled circles. 
Meta-modelling between ontologies is represented by dashed edges. 
Thinnest arrows denote roles within a single ontology while 
thickest arrows denote roles from one ontology to another ontology. 
\\
Figure \ref{fig:secondView}  has  five separate ontologies. 
The ontology in the uppermost position conceptualizes the politics about geographic objects, defining \emph{GeographicObject} as a meta meta-concept, and \emph{Activity} and \emph{GovernmentOffice} as concepts. 
The ontology  in the left middle describes hydrographic objects 
 through the meta-concept \emph{HydrographicObject}  and the one in the right middle describes flora objects  through the meta-concept \emph{FloraObject}.
The two remaining ontologies  conceptualize 
the concrete natural resources at a lower level of granularity 
through  the concepts 
$River$, $Lake$, $Wetland$ and $NaturalForest$. \\ 
Note that horizontal dotted lines in Figure \ref{fig:secondView} do not represent meta-modelling levels
but just ontologies. The ontology ``Geographic Object Politics'' has the meta meta-concept \emph{GeographicObject}, whose instances are concepts which have also instances being concepts, but we also have the concepts \emph{GovernmentOfice} and \emph{Activity} whose instances conceptualize atomic objects. 
\begin{figure}[H]
\centering
\includegraphics[width=0.9\linewidth]{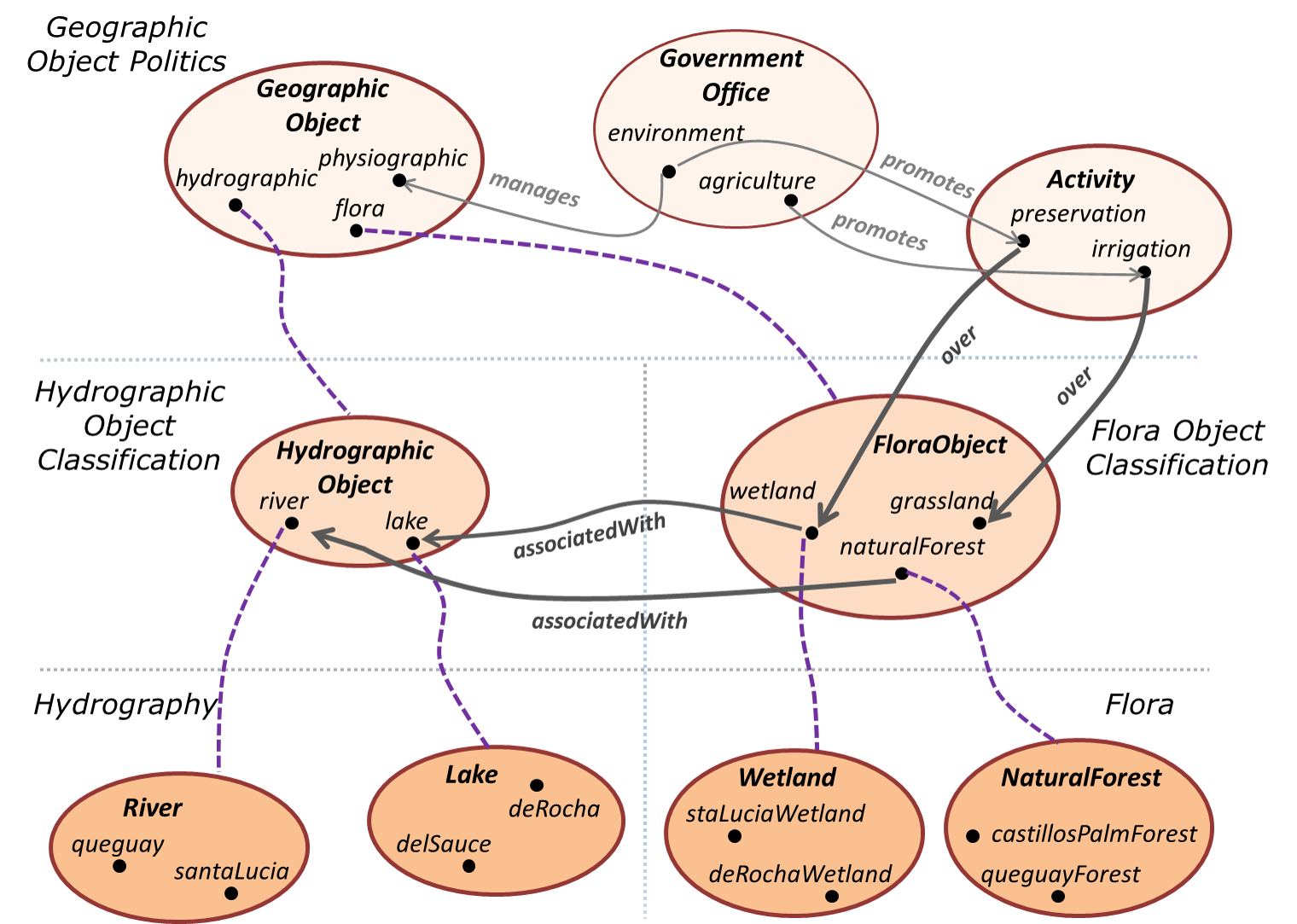}
\caption{Case Study on Geography}
\label{fig:secondView}
\end{figure} 
OWL has only one notion of hierarchy  which classifies concepts
with respect to the  inclusion $\sqsubseteq$. 
Our approach has a new notion of hierarchy, called {\em meta-modelling
hierarchy}, which  classifies concepts
with respect to the membership relation $\in$. 
The meta-modelling hierarchy for the concepts of Figure \ref{fig:secondView}
is depicted in Figure \ref{fig:levelView}.
The concepts are 
 \emph{GovernmentOffice}, \emph{Activity}, 
 \emph{River}, \emph{Lake}, 
 \emph{Wetland} and \emph{NaturalForest}, the meta-concepts are \emph{HydrographicObject} and \emph{FloraObject}, and the meta meta-concept is \emph{GeographicObject}. 
  \\ 
The first advantage of our approach over previous work 
concerns the reuse of ontologies when the same
conceptual object is represented as an individual in one ontology and
as a concept in the other. The identifiers for the individual and the
concept will be syntactically different because they belong to different
ontologies (with different URIs). Then, the ontology engineer
can introduce an equation between these two different identifiers.
This contrasts with   previous approaches where one has to use the same identifier
for an object used as a concept and as an individual.
In Figure \ref{fig:secondView},  $river$ and $River$ represent
the same real object. In order to detect  inconsistency and do
the  proper inferences, one has to be able to equate them. 
\\
The second advantage is about the flexibility of the meta-modelling hierarchy.
This hierarchy is easy to change by just adding equations. 
This is illustrated in the passage from Figure \ref{fig:firstView} to
Figure \ref{fig:secondView}. 
Figure \ref{fig:firstView} has a very simple meta-modelling hierarchy
where the concepts are \emph{River} and  \emph{Lake}
and the  meta-concept is \emph{HydrographicObject}. 
The rather more complex meta-modelling hierarchy for the ontology of 
Figure \ref{fig:secondView}   
(see Figure \ref{fig:levelView}) has been obtained 
 by combining the ontologies 
  of Figure \ref{fig:firstView} with other ontologies and  by simply 
 adding some few meta-modelling axioms.  
%
After adding the meta-modelling equations,
the change of the meta-modelling hierarchy is {\em automatic}
and {\em transparent} to the user.
Concepts such as \emph{GeographicObject} will automatically pass to be 
meta meta-concepts and
roles such as  \emph{associatedWith} 
will automatically pass to be meta-roles, i.e. 
roles between meta-concepts. 
\begin{figure}[H]
\centering
\includegraphics[width=0.7\linewidth]{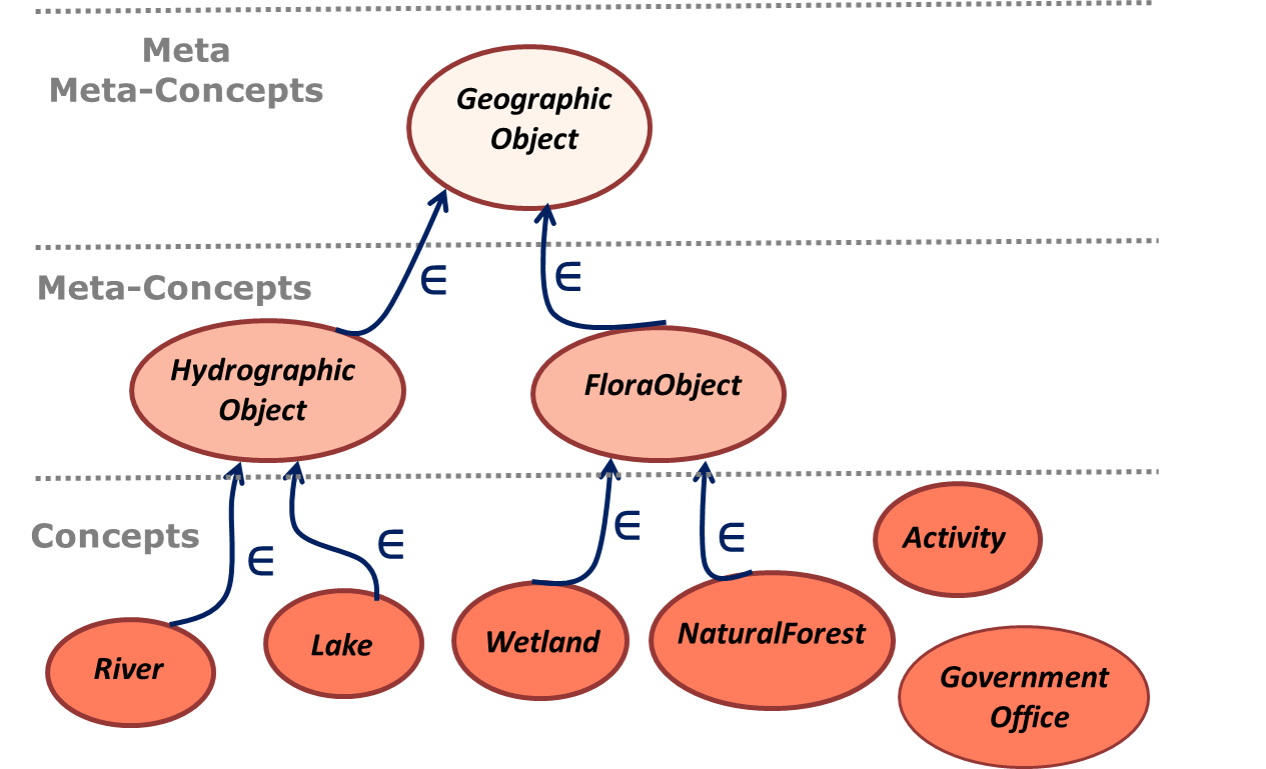}
\caption{Meta-modelling Hierarchy for the Ontology of Figure \ref{fig:secondView}}
\label{fig:levelView}
\end{figure}
The third advantage is that we do not have any restriction on the level of meta-modelling, i.e. we can have concepts, meta-concepts, meta meta-concepts and so on. 
Figure \ref{fig:firstView} has only one level of meta-modelling since there are concepts and meta-concepts. In Figure \ref{fig:secondView}, there are two levels of meta-modelling since
it has concepts, meta-concepts and meta meta-concepts. If we needed, we could extend it further by adding the equation $santaLucia \eqm SantaLucia$ for some concept $SantaLucia$ and this will add a new level in the meta-modelling hierarchy: 
concepts, meta-concepts, meta meta-concepts and meta meta meta-concepts. \\
 Moreover, the user does not have to know the meta-modelling levels, they are transparent for him. 
 Our algorithm  detects  inconsistencies without burdening the user with syntactic complications such as having to explicitly write the level the concept belongs to.\\ 
The fourth advantage is about the possibility of mixing levels
of meta-modelling in the definition of concepts  and roles.   
We can build concepts
using union or intersection between two concepts of different 
levels (layers).  
We can also define  roles whose domain and range live in different levels (or layers).
 For example, in  Figure \ref{fig:secondView}, we have:
1) a role \emph{over} whose  domain 
 is just a concept  while the range is a meta-concept,
2) a role \emph{manages} whose domain is just a concept and whose range
 is a meta meta-concept.
 We  can also add axioms to express
 that some of these concepts, though at different levels of meta-modelling,
 are disjoint, e.g. the intersection of the concept  $\emph{Activity}$ 
 and the meta-concept $\emph{FloraObject}$ is empty.

\section{${\mathcal{ALCQM}}$}
\label{sec:ALCQM}

In this section we introduce the ${\mathcal{ALCQM}}$ Description Logics (DL), with the aim of expressing meta-modelling in a knowledge base. 
The syntax of $\mathcal{ALCQM}$ is obtained from the one of
$\mathcal{ALCQ}$  by adding   new statements 
that allow us to equate individuals with concepts.
The definition of the semantics  for  ${\mathcal{ALCQM}}$ is the key to
our approach. In order to detect inconsistencies coming from meta-modelling,
a proper semantics should give   
 {\em the same interpretation} to individuals and concepts 
which have been equated through meta-modelling.\\
Recall the formal syntax of ${\mathcal{ALCQ}}$ \cite{FOST,DBLP:conf/dlog/2003handbook}.
 We assume a finite set of atomic individuals, concepts and roles. If $A$ is an atomic concept and $R$ is a role, the concept expressions $C$, $D$ are constructed using the following 
 grammar:\\
\begin{tabular}{l}
 $C$, $D$ ::= $A \mid \top \mid \bot  \mid \neg C \mid C \sqcap D \mid C \sqcup D \mid 
 \forall R.C \mid \exists R.C  \mid \geq nR.C \mid \leq nR.C $\
 \end{tabular}\\ 
Recall also that  ${\mathcal{ALCQ}}$-statements are divided in two groups, namely TBox statements and ABox statements, where a TBox contains statements of the form $C \sqsubseteq D$ and an ABox contains statements of the form $C(a)$, $R(a, b)$, $a = b$ or $a \not = b$. \\ 
A {\em meta-modelling axiom} is 
a new type of statement of the form 
\begin{center}
\begin{tabular}{ll}
 $a \eqm A$ & where $a$ is an individual and $A$ is an atomic concept.
 \end{tabular}
 \end{center} 
which we pronounce as {\em $a$ corresponds to $A$ through meta-modelling}.
An {\em Mbox} is a set $\mathcal{M}$ of meta-modelling axioms.
We define ${\mathcal{ALCQM}}$ by keeping the same syntax for concept
expressions as  for  ${\mathcal{ALCQ}}$ and  extending it only to include
MBoxes. 
An ontology or a knowledge base in ${\mathcal{ALCQM}}$ is denoted by 
$\mathcal{O} = (\mathcal{T}, \mathcal{A}, \mathcal{M})$ since it is determined by three sets:
 a Tbox $\mathcal{T}$,  an Abox $\mathcal{A}$ and an Mbox $\mathcal{M}$. The set of all individuals with meta-modelling of an ontology is denoted
by $\dom (\mathcal M)$.\\ 
Figure \ref{figure:boxesforfirstview} 
shows the $\mathcal{ALCQM}$-ontologies of Figure \ref{fig:firstView}.
%
In order to check for cycles in the tableau algorithm, it is convenient 
to have the restriction that $A$ should be a concept name in  $a \eqm A$.
This restriction does not affect us in practice at all. If one would like to
have $a \eqm C$ for a concept expression $C$, it is enough to introduce
a concept name $A$ such that $A \equiv C$ and $a \eqm A$.
\begin{figure}
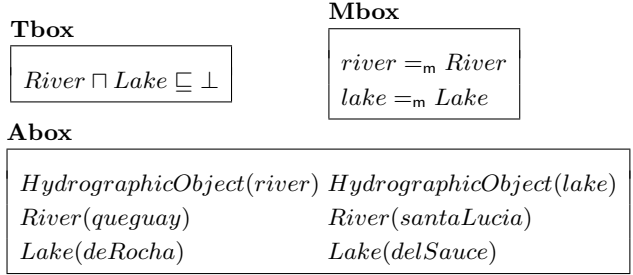

{\small
\begin{center}
\begin{tabular}{l}
\begin{tabular}{lr}
\begin{tabular}{l}
{\bf Tbox} \\
\begin{tabular}{|l|}
\hline   \\[-0.8ex]
\ $River \sqcap Lake \sqsubseteq \bot $ \ \\[0.5ex]
\hline 
\end{tabular}
\end{tabular}
& \ \ \ \ \ \ \ \ \ \ 
\begin{tabular}{l}
{\bf Mbox} \\
\begin{tabular}{|l|}
\hline  \\[-0.8ex]
\ $river \eqm  River$ \ \\[0.5ex]
\ $lake \eqm Lake$ \ \\[0.5ex]
\hline
\end{tabular}
\end{tabular}
\end{tabular}
\\
\begin{tabular}{l}
{\bf Abox} \\
\begin{tabular}{|ll|}
\hline & \\[-0.8ex]
\ $HydrographicObject(river)$ & $HydrographicObject(lake)$ \ \\[0.5ex]
\ $River(queguay)$ & $River(santaLucia)$\ \\[0.5ex]
\ $Lake(deRocha)$ & $Lake(delSauce)$ \ \\[0.5ex]
\hline
\end{tabular}
\end{tabular}
\end{tabular}
\end{center}
}
\caption{The $\mathcal{ALCQM}$-ontology of  Figure \ref{fig:firstView}}
\label{figure:boxesforfirstview} 
\end{figure}
%
%

\begin{definition}[$S_n$ for $n \in \mathbb{N}$]
Given a non empty set $S_{0}$ of atomic objects, we define $S_{n}$ by induction on $\mathbb{N}$ as follows:
$S_{n+1} = S_{n} \cup \mathcal{P}(S_{n}) $
\end{definition}
\noindent
The sets $S_{n}$ are clearly well-founded. Recall from Set Theory that {\em a relation $R$ is well-founded on a class $X$} if every non-empty subset $Y$ of $X$ has a minimal element. Moreover, {\em a set $X$ is well-founded} if the set membership relation is well-founded on the the set $X$.

\begin{definition}[Model of an Ontology in $\mathcal{ALCQM}$]
\label{definition:modelALCQM}
An interpretation ${\mathcal{I}}$ 
is a \emph{model of an ontology 
$\mathcal{O} = (\mathcal{T}, \mathcal{A}, \mathcal{M})$ 
in ${\mathcal{ALCQM}}$ } 
(denoted as $\mathcal{I} \models \mathcal{O}$)
 if the following holds:

\begin{enumerate}

\item the domain 
 $\Delta$ of the interpretation is a subset  of $S_{N}$ 
  for some  $N \in \mathbb{N}$.
 The smallest   $N$ such that $\Delta \subseteq S_N$ is called the 
 \emph{level}  of the interpretation $\mathcal{I}$.

  \item  ${\mathcal{I}}$ is a model of the ontology 
$(\mathcal{T}, \mathcal{A})$   
  in $\mathcal{ALCQ}$.
  
  \item ${\mathcal{I}}$ is a model of $\mathcal{M}$, i.e.
${\mathcal{I}}$   satisfies  each statement in $\mathcal{M}$.
An interpretation $\mathcal{I}$ satisfies the statement $a \eqm A$
if $a^{\mathcal{I}} = A^{\mathcal{I}}$. 
\end{enumerate}
 
\end{definition}

\noindent
Usually, the domain of an interpretation of an ontology is a set of atomic objects. In the first part of Definition \ref{definition:modelALCQM} we redefine the domain $\Delta$
 of the interpretation, so it does not consists  only of  atomic objects any longer. 
The domain $\Delta$ can now contain  sets  since 
 the set $S_N$  is defined recursively using
the power-set operation.
 A similar notion  of  interpretation domain  is defined    
  in \cite[Definition 1]{conf/fois/KaushikFWA06}  for RDF ontologies. 

It is sufficient to require that it is a subset of some $S_{N}$ so it remains well-founded
\footnote{In principle, non well-founded sets are not source of contradictions since we could work on non-well founded Set Theory. 
The reason why we exclude them is because we think that non well-founded sets do not occur in the applications we are interested in. 
}. 
Note that $S_0$ does not have to be the same for all models of an ontology. The second part of Definition \ref{definition:modelALCQM} refers to the $\mathcal{ALCQ}$-ontology without the Mbox axioms. In the third part of the definition, we add another condition that the model must satisfy considering the meta-modelling axioms. This condition restricts the interpretation of an individual that has a corresponding concept through meta-modelling  to be equal to the concept interpretation.

\begin{example}
We define a model for the ontology of Figure \ref{figure:boxesforfirstview} where 
\[ S_0 = \{ queguay, santaLucia, deRocha, delSauce \}\]
Individuals and concepts equated through meta-modelling are semantically equal:
\[
\begin{array}{lll}
river^{\mathcal{I}} & = River^{\mathcal{I}} & = \{queguay, santaLucia\} \\
lake^{\mathcal{I}} & = Lake^{\mathcal{I}}  & = \{deRocha, delSauce\}
\end{array}
\]
\end{example}

\begin{definition}[Consistency of an Ontology in $\mathcal{ALCQM}$]
We say that an ontology  
$\mathcal{O} =(\mathcal{T}, \mathcal{A}, \mathcal{M})$ is \emph{consistent} if there exists a model of $\mathcal{O}$.
\end{definition}
The $\mathcal{ALCQM}$-ontology defined in Figure \ref{figure:boxesforfirstview}
is consistent.

\begin{example}
\label{example:transferenceequality}
We consider the ontology  of Figure \ref{fig:secondView} 
and add the axiom  
\[\begin{array}{ll}
Wetland \equiv NaturalForest 
 \\
\end{array}
\]
 and the fact that  $associatedWith$ is a functional property.
 Note that we have the following axioms in the Abox:
\begin{center}
  \begin{tabular}{l}
  $associatedWith(wetland, lake)$\\ 
  $associatedWith(naturalForest, river)$
  \end{tabular}
\end{center}
As before, the $\mathcal{ALCQ}$-ontology  (without the Mbox)   is consistent.
  However, the $\mathcal{ALCQM}$-ontology  (with the Mbox) is not consistent. 
\end{example}
Example \ref{example:five} illustrates the use of the first clause
of Definition \ref{definition:modelALCQM}. Actually, this example is inconsistent
because the first clause of this definition does not hold. 
Examples \ref{example:simpletransference} and \ref{example:transferenceequality}
  illustrate how the second and third conditions of Definition \ref{definition:modelALCQM} interact.

\begin{definition}[Logical Consequence from an Ontology in $\mathcal{ALCQM}$]
We say that $\judgement$ is a \emph{logical consequence of} 
$\mathcal{O} =(\mathcal{T}, \mathcal{A}, \mathcal{M})$ 
(denoted as $\mathcal{O} \models \judgement$) if all models of $\mathcal{O}$ are also models  
of $\judgement$ where $\judgement$ 
is any of the following  $\mathcal{ALCQM}$-statements, i.e.
 $C \sqsubseteq D$, $C(a)$, $R(a, b)$,  $a \eqm A$, $a = b$   and $a \not= b$.  
\end{definition}
It is  possible to infer new knowledge in the ontology
with the meta-modelling that is not possible without it
as illustrated by Examples 
\ref{example:five}, \ref{example:simpletransference} and \ref{example:transferenceequality}.

%

\begin{definition}[Meta-concept]
\label{definition:metaconcept}
We say that $C$ is a meta concept in $\mathcal{O}$ if there exists an individual $a$ such that $\mathcal{O} \models  C(a)$ and $\mathcal{O} \models a \eqm A$.
\end{definition}

Then, $C$ is a meta meta-concept if there exists an individual $a$ such that 
$\mathcal{O} \models  C(a)$, $\mathcal{O} \models a \eqm A$ and $A$ is a meta-concept.
Note that a meta meta-concept is also a meta-concept.

\noindent
We have some new inference problems:

\begin{enumerate}
\item {\em Meta-modelling}. 
Find out whether $\mathcal{O} \models a \eqm A$ or not.
\item {\em Meta-concept}. Find out whether $C$ is a meta-concept or not.
\end{enumerate}

\noindent
Most inference problems in Description Logic can be reduced to satisfiability by applying 
a standard result in logic which says that a formula 
$\phi$ is a semantic consequence
of a set of formulas $\Gamma$  
if and only if $\Gamma \cup \neg \phi$ is not satisfiable.
The first two problems can be reduced to satisfiability  following this general idea.
For the first problem, note that 
 since  $a \not =_m A$ is not directly available in the
 syntax, we have replaced it by $a \not = b $  and $ b \eqm A$
 which is  an equivalent statement 
to the negation of  $a \eqm A$ 
 and  can be expressed in $\mathcal{ALCQM}$.
 
\begin{lemma}
\label{lemma:firstreasoningproblem}
$\mathcal{O} \models a \eqm A$ if and only if 
for some new individual $b$, $\mathcal{O} \cup \{a \not = b, b \eqm A \}$ is unsatisfiable.
\end{lemma}

\begin{lemma}
\label{lemma:secondreasoningproblem}
$C$ is a meta-concept
if and only if
for some  individual $a$ we have that
 $\mathcal{O} \cup \{\neg C(a) \}$  is  unsatisfiable
and
for some new individual $b$, $\mathcal{O} \cup \{a \not = b, b \eqm A \}$ is unsatisfiable.
\end{lemma}

\section{Checking Consistency of an Ontology in   $\mathcal{ALCQM}$}
\label{sec:ReasonerALCQM}

In this section we will define a tableau 
algorithm for checking consistency of  an ontology
 in $\mathcal{ALCQM}$ by extending 
the tableau algorithm for 
$\mathcal{ALCQ}$. 
From the practical point of view, 
extending  tableau for $\mathcal{ALCQ}$
 has the advantage that one can easily  change and reuse the 
code of existing OWL's reasoners. \\
The tableau algorithm for $\mathcal{ALCQM}$ is defined 
by adding  three  expansion rules and a  condition to the tableau algorithm for
 $\mathcal{ALCQ}$. 
The new  expansion rules  deal  with the equalities and inequalities between individuals
 with meta-modelling  which need to be transferred to the level of concepts
 as equalities and inequalities between the corresponding concepts.
 The new condition  deals  with circularities avoiding sets that belong to
 themselves and more generally, avoiding non well-founded sets.

\begin{definition}[Cycles]
\label{definition:cycle}
We say that the tableau graph $\lab$ has a cycle with respect to $\mathcal{M}$ if 
 there exist a sequence of meta-modelling axioms $ A_0 \eqm a_0$, $ A_1  \eqm a_1 $, $ \ldots$ $ A_n \eqm  a_n$ all in $\mathcal{M}$ such that 
\[\begin{array}{ll}
A_1 \in \lab(x_0) & x_0 \eql a_0\\
A_2 \in \lab(x_1) & x_1 \eql a_1\\
\vdots & \vdots \\
A_n \in \lab(x_{n-1}) \ \ \ & x_{n-1} \eql a_{n-1}\\
A_0 \in \lab(x_{n}) & x_{n} \eql a_{n}
\end{array}
\]  
\end{definition}

\begin{example}
Suppose we have an ontology 
$(\mathcal{T},  \mathcal{A}, \mathcal{M})$  with 
two individuals $a$ and $b$, the  individual assignments:
$B(a)$ and  
$A(b)$;
and the meta-modelling axioms:
\begin{center}
$ a \eqm A$ \ \ \ $b \eqm  B$. 
\end{center}
The tableau graph $\lab(a) = \{B \}$ and $\lab(b) =\{A \}$ has a cycle since 
$A \in \lab(b)$ and $B \in \lab(a)$.
\end{example}

Initialization for the $\mathcal{ALCQM}$-tableau is nearly 
the same as for $\mathcal{ALCQ}$. The nodes of the initial tableau graph
 will be created from individuals that occur in the Abox as well as
 in the Mbox.
After initialization, the tableau algorithm proceeds by non-deterministically applying the 
 {\bfseries{expansion rules}}  for $\mathcal{ALCQM}$.
The expansion rules for $\mathcal{ALCQM}$ are obtained by adding the
  rules of Figure \ref{figure:newrules} to  the expansion rules for $\mathcal{ALCQ}$. 

\begin{figure}
{\small
\begin{tabular}{ll}
{\bf $\eql$-rule}:&
Let $a \eqm A$ and $b \eqm B$ in $\mathcal{M}$. 
If $a \eql b$ 
and $A \sqcup \neg B, B \sqcup \neg A$ does not \\
& belong to $\mathcal{T}$ then ${\cal T} \leftarrow A \sqcup \neg B, B \sqcup \neg A$. \\
{\bf $\neql$-rule}:& 
Let $a \eqm A$ and $b \eqm B$ in $\mathcal{M}$. 
If $a \neql b$ 
and there is no $z$ such that \\
& $A \sqcap \neg B \sqcup B \sqcap \neg A \in \lab(z)$  
then  create a new node $z$ with \\
& $\lab(z) = \{ A \sqcap \neg B \sqcup B \sqcap \neg A \}$. \\
{\bf close-rule}:& 
Let $a \eqm A$ and $b \eqm B$ where
 $a \eql x$, $b \eql y$, $\lab(x)$ and  $\lab(y)$ are  defined.\\
 & If  neither $x \eql  y$ nor $x  \neql y$  are set  
then $\equate{a}{b}{\lab}$ or $\makedifferent{a}{b}{\lab}$ .
\end{tabular}
}
\caption{Additional Expansion Rules for $\mathcal{ALCQM}$}
\label{figure:newrules}
\end{figure}

We explain the intuition  behind the new expansion rules.
If $a \eqm A$ and $b \eqm B$ then the individuals $a$ and $b$ 
represent concepts. Any equality at the level of
individuals should be transferred as an equality between concepts
and similarly with the difference. \\ 
The $\eql$-rule transfers  the equality $a \eql b$  to the level of concepts
by adding  two statements to the Tbox which are
equivalent to $A \equiv B$.
 This rule is necessary to detect the inconsistency of 
Example \ref{example:simpletransference} where the equality $river = lake$
is transferred as an equality $River \equiv Lake$ between concepts.
A particular case of the application of the $\eql$-rule  is
when $a \eqm A$ and $a \eqm B$. In this case, the algorithm also adds
$A \equiv B$.  \\ 
The $\neql$-rule is similar to the $\eql$-rule.
However, in the case that  $a \neql b$, we cannot add
$A \not \equiv B$ because the negation of $\equiv$ is not
directly  available in
the language. So, what we do is to 
 replace it by an equivalent statement, i.e. add an element $z$ that witness 
this difference. \\ 
The rules $\eql$ and $\neql$ are not sufficient to detect all
inconsistencies.
With only these rules, we could not detect the inconsistency of
Example \ref{example:transferenceequality}.
 The idea is that  we also need to
  transfer the equality $A \equiv B$ between concepts as an equality
  $a \eql b$ between individuals.
 However, here we face a delicate problem.
 It is not enough to transfer the equalities
 that are in the Tbox. We also need to transfer 
 the semantic consequences, e.g. $\mathcal{O} \models A \equiv B$.
 Unfortunately, we cannot do $\mathcal{O} \models A \equiv B$.
Otherwise we will be captured in a vicious circle~\footnote{Consistency is
the egg and semantic consequence is the chicken.}  
since the problem of finding out the semantic consequences is
reduced to the one of satisfiability. 
The  solution to this  problem is to explicitly try
either $a \eql b$ or $a \neql b$. This is exactly what the
close-rule does. 
The close-rule adds either $a \eql b$ or $a \neql b$.
It is similar to the choose-rule which adds either $C$ or $\neg C$. 
This works because we are working in Classical Logic
and we have the law of excluded middle. For a model $\int$ of the ontology, we have that either
$a^{\int} = b^{\int}$ or $a^{\int} \not = b^{\int}$
(see also Lemma \ref{lemma:modelstar}). 
Since the tableau algorithm works with canonical representatives of
the $\eql$-equivalence classes, we have to be careful how we 
equate two individuals or make them different. 
\\
Note that the application of the tableau algorithm to an $\mathcal{ALCQM}$ knowledge base $(\mathcal{T},  \mathcal{A}, \mathcal{M})$ changes the Tbox as well as the tableau graph $\lab$.

\begin{definition}[${\cal ALCQM}$-Complete]
$(\mathcal{T}, \lab)$ is ${\cal ALCQM}$-complete if 
none of the expansion rules for ${\cal ALCQM}$
 is applicable.
\end{definition}

The algorithm terminates when we reach some $(\mathcal{T}, \lab)$ 
where  either $(\mathcal{T}, \lab)$ is 
${\cal ALCQM}$-complete, $\lab$ has a contradiction or $\lab$ has a cycle. 
 The ontology $(\mathcal{T}, \mathcal{A}, \mathcal{M})$ is consistent if 
 there exists some 
 ${\cal ALCQM}$-complete $(\mathcal{T},\lab)$ 
 such that $\lab$ has neither contradictions nor cycles. Otherwise it is inconsistent.

\section{Correctness of the Tableau Algorithm for $\mathcal{ALCQM}$}
\label{sec:CorrectTableau}

In this section we prove termination, soundness
and completeness 
 for the tableau algorithm described in the previous section.
We give an elegant proof of
 completeness  by showing an isomorphism between the canonical interpretations
 of $\mathcal{ALCQ}$ and $\mathcal{ALCQM}$.

\begin{theorem}[Termination]
\label{theorem:termination}
The tableau algorithm for $\mathcal{ALCQM}$ described in the previous section
always terminates.
\end{theorem}

\begin{proof}
Suppose the input is an arbitrary ontology 
$\mathcal{O}  = (\mathcal{T}, \mathcal{A}, \mathcal{M})$.
We define
\[
\begin{array}{ll}
\boundM &  = \bigcup_{a \eqm A, b \eqm B}  
 \{ A \sqcap \neg B \sqcup B \sqcap \neg A, A \sqcup \neg B, B \sqcup \neg A \}
\end{array}
\]
Suppose we have an infinite sequence of rule applications:
\begin{equation}
\label{equation:infinitesequence}
(\mathcal{T}_0, \lab_0) \Rightarrow 
(\mathcal{T}_1, \lab_1 ) \Rightarrow 
(\mathcal{T}_2, \lab_2) \Rightarrow  \ldots
\end{equation}
where $\Rightarrow$ denotes the application of one $\mathcal{ALCQM}$-expansion rule.
In the above sequence, the number of applications of the $\eql, \neql$ and
{\sf close}-rules is finite as we show below:

\begin{enumerate}
\item 
 The $\eql$ and $\neql$-rules can  be applied only a finite number of times in the above
sequence. 
The  $\eql$ and $\neql$-rules add concepts to the Tbox and 
these  concepts that can be added  all belong to 
$\boundM$ which is finite. We also have that $\mathcal{T}_i \subseteq 
\mathcal{T} \cup \boundM$ for all $i$.
 Besides none of the other rules remove elements from the Tbox.
 
\item Since the set 
$\{ (a,b) \mid a, b \in \dom(\mathcal{M})\}$ 
 is finite, the {\sf close}-rule can 
 be applied only  a finite number of times. This is because once
 we set $a \eql b$ or $a \neql b$, no rule can ``undo'' this.
\end{enumerate}
This means that from some $n$ onwards in  sequence  (\ref{equation:infinitesequence})
 \begin{equation}
 \label{equation:secondinfinitesequence}
(\mathcal{T}_n, \lab_n) \Rightarrow
(\mathcal{T}_{n+1}, \lab_{n+1} ) \Rightarrow 
(\mathcal{T}_{n+2}, \lab_{n+2}) \Rightarrow \ldots
\end{equation}
there is  no application
of the rules $\eql$, $\neql$ and {\sf close}. Moreover, 
$\mathcal{T}_n = \mathcal{T}_{i}  $ for all $i \geq n$.
Now,  sequence (\ref{equation:secondinfinitesequence})
contains only application of  $\mathcal{ALCQ}$-expansion rules.
This sequence is  finite 
by \cite[Proposition 5.2]{journals/ai/BaaderBH96}.
This is a contradiction.
\end{proof}

The proof of the following theorem is similar to Soundness for
$\mathcal{ALCQM}$ \cite{journals/ai/BaaderBH96}..

\begin{theorem}[Soundness]
\label{theorem:soundess}
If
$\mathcal{O} = (\mathcal{T}, \mathcal{A}, \mathcal{M})$
is  consistent then
the $\mathcal{ALCQM}$-tableau graph terminates and yields
an $\mathcal{ALCQM}$-complete $(\mathcal{T}_k, \lab_k)$
such  that $\lab_k$ has neither
cycles nor contradictions.
\end{theorem}

The following definition of canonical interpretation is basically 
 the one in \cite[Definition 4.3]{journals/ai/BaaderBH96}.
Instead of  $<$, we use the idea of descendants.

\begin{definition}[$\mathcal{ALCQ}$-Canonical Interpretation]
\label{definition:intcan}
We define the 
$\mathcal{ALCQ}$-canonical interpretation 
$\intcan$ from  a tableau graph $\lab$ as follows.
\[\begin{array}{lll}
\Deltacan & = & \{ x \mid \lab(x) \mbox{ is defined} \} \\
(x)^{\intcan} &  = & \left \{ \begin{array}{ll}
x & \mbox{ if } x \in \Delta^{\intcan} \\
y & \mbox{ if } x \eql y \mbox{ and } y \in \Delta^{\intcan}  
\end{array} \right . \\
(A)^{\intcan} & = & \{ x \in \Delta^{\intcan} \mid A \in \lab(x) \}\\
 (R)^{\intcan} & = 
& \{ (x, y) \in \Delta^{\intcan}  \times \Delta^{\intcan} \mid 
 R \in \lab(x, y) \mbox{ $x$ is not blocked or } \\  
& &  R \in \lab(z,y) \mbox{ where } x 
\mbox{ is blocked by $z$ and $z$ is not blocked }\} 
\end{array}
\]
\end{definition}

%

Note that the canonical interpretation is not defined on equivalence classes of $\eql$
but by choosing  canonical representatives.

\begin{lemma}
\label{lemma:modelintcan}
If 
the tableau algorithm for $\mathcal{ALCQM}$ 
with input $\mathcal{O}= (\mathcal{T}, \mathcal{A}, \mathcal{M})$
yields a $\mathcal{ALCQM}$-complete
 $(\mathcal{T'}, \lab)$ such that $\lab$ has no contradictions
 then
 $\intcan$ is a model of 
 $(\mathcal{T}, \mathcal{A})$. 
\end{lemma}

\begin{proof} 
We define  ${\sf rel} (\aboxL)$ as follows.
   \[
 \begin{array}{ll}
 {\sf rel} (\aboxL) = 
 & \{ C(x) \mid C \in \lab(y ), y \eql x \} \cup \\ 
          &      \{R(a,b) \mid R \in \lab(x,y), a \eql x, b \eql y                
               \{ a, b\} \subseteq \mathcal{O}  
               \} \cup  \\
          &      \{ x = y \mid x \eqL y  \} \cup
                \{ x \not = y \mid x \not \eqL y  \}
\end{array} 
 \]
By  \cite[Lemma 5.5]{journals/ai/BaaderBH96}),
$\intcan$ is a model of $(\mathcal{T'},  {\sf rel} (\aboxL) )$.
Since $\mathcal{T} \subseteq \mathcal{T'}$ and
$\mathcal{A} \subseteq {\sf res} (\aboxL)$, we have that
$\intcan$ is a model of $(\mathcal{T}, \mathcal{A})$.
\end{proof}

So, how can we now make $\intcan$ into a model of the whole ontology $(\mathcal{T}, \mathcal{A}, \mathcal{M})$? We will transform $\intcan$ into a model of 
$(\mathcal{T}, \mathcal{A}, \mathcal{M})$ by defining a function $\unfold$. 
The following lemma allows us to give a recursive definition of $\unfold$.
 
\begin{lemma}
\label{lemma:wellfounded}
If the tableau graph $\lab$ has no cycles then  
$(\Deltacan, \prec)$ is well-founded 
where $\prec$ is the relation defined as $y \prec x$ if
$y \in (A)^{\intcan}$, $x \eql a$ and $a \eqm A \in \mathcal{M}$.
\end{lemma}

\begin{proof}
Suppose $(\Deltacan, \prec)$  is not well-founded. Since $\Deltacan$ 
is finite, infinite descendent $\prec$-sequences can only be formed from
$\prec$-cycles, i.e. they are of the form 
 \[ y_n \prec y_1 \prec  \ldots  \prec y_n  \]
  It is easy to see that 
 this contradicts the fact
 that $\lab$ has no cycles. 
\end{proof}
%

\begin{definition}[From Basic Objects to Sets: the function $\unfold$]
 \label{definition:unfold}
Let $\lab$ a tableau graph without cycles and 
$\intcan$ be the $\mathcal{ALCQ}$-canonical interpretation from
 $\lab$. 
 For  $x \in \Delta^{\intcan}$  we define $\unfold(x)$ as follows.
\[
\begin{array}{lll}
\unfold(x) & = \{ \unfold(y) \mid y \in (A)^{\intcan} \}
  & \mbox{ if $x \eql a$ for some  $a \eqm A \in \mathcal{M}$ }\\ 
\unfold(x) & = x & \mbox{otherwise}
\end{array}
\]
\end{definition}
\begin{lemma}
\label{lemma:modelstar}
Let $\lab$ be a $\mathcal{ALCQM}$-complete tableau graph without contradictions.
If $a \eqm A$ and $a' \eqm A'$
then
either $a \eql  a'$ or $ a \not \eql a' $.
In the first case,
 $A^{\intcan} = A'^{\intcan}$ and in the second case,
  $A^{\intcan}  \not =  A'^{\intcan}$
\end{lemma}

\begin{lemma}
\label{lemma:unfoldwelldefined}
Let $\lab$ be a $\mathcal{ALCQM}$-complete tableau graph that has neither
contradictions nor cycles  and let 
$\intcan$ be  the   canonical interpretation from $\lab$.
Then,
 $\unfold$ is an injective  function, i.e.
$x = x'$ if and only if $\unfold(x) = \unfold(x')$. 
\end{lemma}

\begin{proof}
We prove first that $\unfold$ is a function.
It is enough to consider the case 
when $x \eql a \eqm A$ and  $ x \eql a' \eqm A'$.
By Lemma \ref{lemma:modelstar}, $a \eql a'$ and
 $(A)^{\intcan} =(A')^{\intcan}$. Hence,
$\unfold(x)$ is uniquely determined.

To prove that $\unfold$ is injective, 
we do induction on $(\Deltacan, \prec)$ which we know that is well-founded by
Lemma \ref{lemma:wellfounded}.
By Definition of $\unfold$, we have two cases.
The first case  is when $\unfold(x) = x$. 
We have that  $\unfold(x') = x$ and $x'$ is exactly $x$.
This was the base case. 
In the second case, we have that
for $x \eql a$ and $a \eqm A$,
\[ \unfold(x) =
\{ \unfold(y) \mid y \in (A)^{\intcan} \}
\]
Since $\unfold(x) = \unfold(x')$, we also have
that $x' \eql a'$ and $a'\eqm A'$ such that
\[ \unfold(x') =
\{ \unfold(y') \mid y' \in (A')^{\intcan} \}
\]
Again since $\unfold(x) = \unfold(x')$,
we have that $\unfold(y) = \unfold(y')$.
By Induction Hypothesis, $y = y'$ for all $y \in (A)^{\intcan} $.
Hence, $(A)^{\intcan} \subseteq (A')^{\intcan}$.
Similarly, we get $(A')^{\intcan} \subseteq (A)^{\intcan}$.
So, $(A)^{\intcan} = (A')^{\intcan}$.
It follows from Lemma \ref{lemma:modelstar} that
$a \eql a'$. Then, $x = x'$ because the canonical representative
of an equivalence class is unique.
\end{proof}
 We are now ready to define the canonical interpretation
 for an ontology in $\mathcal{ALCQM}$.

\begin{definition}[Canonical Interpretation for $\mathcal{ACLQM}$]
\label{definition:intext}
Let $\lab $ be an $\mathcal{ALCQM}$-complete tableau graph without cycles
and without contradictions.
We define  the canonical interpretation $\mathcal{I}^{m}$
for $\mathcal{ALCQM}$  as follows:
\[
\begin{array}{ll}
\Deltaext & = \{ \unfold(x) \mid x \in \Deltacan \} \\
(a)^{\intext} & = \unfold(a) \\
(A)^{\intext}  & = \{ \unfold(x) \mid x \in A^{\intcan} \}\\
(R)^{\intext} & = \{ (\unfold(x), \unfold(y)) \mid (x,y) \in (R)^{\intcan} \}
\end{array}
\]
\end{definition}

\begin{definition}[Isomorphism between interpretations of
$\mathcal{ALCQ}$]\\
An isomorphism between two 
interpretations $\int $ and $\int'$ of $\mathcal{ALCQ}$
is a bijective function $f: \Delta \rightarrow \Delta'$
such that
\begin{itemize}
\item
$f(a^{\int}) = a^{\int'}$
\item
$x \in A^{\int}$ if and only if $f(x) \in A^{\int'}$

\item
$(x,y) \in R^{\int}$ if and only if
$(f(x), f(y)) \in R^{\int'}$.
 
\end{itemize}

\end{definition}

\begin{lemma}
\label{lemma:isomorphism}
Let $\int$ and $\int'$ be two isomorphic interpretations of $\mathcal{ALCQ}$.
Then, 
$\int$ is a model of $(\mathcal{T}, \mathcal{A})$
if and only if
$\int'$ is a model of $(\mathcal{T}, \mathcal{A})$.
\end{lemma}
To prove the previous lemma is enough to show that $x \in C^{\int}$ if and only if
$f(x) \in C^{\int'}$ by induction on $C$.

\begin{theorem}[Completeness]
\label{theorem:completeness}
If $(\mathcal{T}, \mathcal{A}, \mathcal{M})$ is not consistent
then   the $\mathcal{ALCQM}$-tableau algorithm
with input $(\mathcal{T}, \mathcal{A}, \mathcal{M})$
terminates and yields  an $\mathcal{ALCQM}$-complete $(\mathcal{T'}, \lab)$
such that  
$\lab$ that has either a contradiction or a cycle.
\end{theorem}
\begin{proof} 
By Theorem \ref{theorem:termination}, the $\mathcal{ALCQM}$-tableau algorithm
with input $(\mathcal{T}, \mathcal{A}, \mathcal{M})$
terminates.  Suppose towards a contradiction that
 the algorithm  yields  an $\mathcal{ALCQM}$-complete  
$(\mathcal{T'},\lab)$ such that  that $\lab$
has neither a contradiction nor a cycle.  
We will prove that 
$\intext$ is a model of $(\mathcal{T}, \mathcal{A}, \mathcal{M})$.
For this we have to check that $\intext$ satisfies the 
three conditions of Definition \ref{definition:modelALCQM}.
\begin{enumerate}
\item 
In order to  prove that $\Deltaext \subseteq S_N$ for some $S_N$ and $N$,
we define
$S_0 = \{ x \in \Deltacan \mid \unfold(x) = x \}$.

\item We now prove that $\intext$ is a model of $(\mathcal{T}, \mathcal{A})$.
By Lemma \ref{lemma:modelintcan},
the canonical interpretation $\intcan$ is a model of
$(\mathcal{T}, \mathcal{A})$.
It follows from  Lemma \ref{lemma:unfoldwelldefined} that
$\unfold: \Deltacan \rightarrow \Deltaext$ is a bijective map.
It is also easy to show that 
$\intcan$ and $\intext$ are isomorphic interpretations
in $\mathcal{ALCQ}$.
By Lemma \ref{lemma:isomorphism},
$\intext$ is a model of $(\mathcal{T}, \mathcal{A})$.

\item  Finally, we  prove
that $a^{\intext} = (A)^{\intext}$ for all $ a \eqm A \in \mathcal{M}$. 
Suppose that $ a \eqm A \in \mathcal{M}$. 
Then,
\[
\begin{array}{lll}
a^{\intext} & = \unfold(a) & \mbox{ by Definition \ref{definition:intext}} \\
            & = \{ \unfold(x) \mid x \in (A)^{\intcan} \} & \mbox{ by Definition \ref{definition:unfold}} \\
            & = A^{\intext} & \mbox{ by Definition \ref{definition:intext}}
\end{array}
\]
\end{enumerate}
\end{proof}

A direct  corollary from the above result is that
$\mathcal{ALCQM}$ satisfies the finite model property.

\section{Conclusions and Future Work}
\label{sec:relWorks}

 In this paper we present a tableau algorithm for checking consistency of an ontology in $\mathcal{ALCQM}$ and prove its correctness. 
In order to implement our algorithm, we plan to incorporate 
optimization techniques such as normalization, absorption
or the use of heuristics \cite[Chapter9]{DBLP:conf/dlog/2003handbook}.  \\
A first step to optimize the algorithm would be to
impose the following order on the application of the expansion rules.
We apply the  rules that create nodes ($\exists$ and $\geq$)
 only if the other rules are
not applicable. We apply the bifurcating 
rules  ($\sqcup$, {\sf choose} or {\sf close}-rules)
if the remaining rules (all rules except the $\exists$, $\geq$, $\sqcup$,
{\sf choose} and {\sf close}-rules) are not applicable.
One could prove that this strategy is correct similarly to 
 Section \ref{sec:CorrectTableau}.\\
 A second step to optimize the algorithm would be to change
the $\eql$-rule. 
 Instead of adding  $A \sqcup \neg B$ and $\neg A \sqcup B$,
we could add $A \equiv B$  and treat this as a trivial case
of lazy unfolding.\\
We would also like to study decidability of consistency 
 for the kind of meta-modelling presented in this paper
in  more powerful Description Logics than $\mathcal{ALCQM}$.\\
We believe that consistency in $\mathcal{ALCQM}$
 has the same complexity as $\mathcal{ALCQ}$, 
which is Exp-time complete \cite{TobiesPhDThesis2001}. 
We also plan to study worst-case
optimal tableau algorithms for $\mathcal{ALCQM}$
 \cite{DBLP:journals/ai/DoniniM00,DBLP:conf/dlog/GoreN07}.

\paragraph{Acknowledgements.}
We are grateful to Diana Comesa\~{n}a for sharing with us
the data from the ontology network on geographic objects she is
developing in Uruguay  \cite{bworld}.

%

\appendix

\section{Example of $\mathcal{ALCQM}$-ontology}

Figure \ref{figure:boxes}
shows the $\mathcal{ALCQM}$-ontology of Figure \ref{fig:secondView}.

\begin{figure}
{\small 
\begin{center}
\begin{tabular}{lr}
{\bf Tbox} \\
\begin{tabular}{|ll|}
\hline  &
\\[-0.8ex]
$GovernmentOffice \sqsubseteq \exists manages.GeographicObject$ &\\ 
$Activity \sqsubseteq \forall over.(HydrographicObject \sqcup FloraObject)$ &\\
$FloraObject \sqsubseteq \forall associatedWith.HydrographicObject$ &\\
$River \sqcap Lake \sqsubseteq \bot $ &\\[0.5ex]
\hline
\end{tabular}
\end{tabular} 
\end{center}

\begin{center}
\begin{tabular}{lr}
{\bf Abox} \\
\begin{tabular}{|ll|}
\hline & 
\\[-0.8ex]
$GeographicObject(hydrographic)$ & $GeographicObject(physiographic)$ \\
$GeographicObject(flora)$ &\\
$GovernmentOffice(environment)$ & $GovernmentOffice(agriculture)$\\
$Activity(preservation)$ & $Activity(irrigation)$\\
$manages(environment, physiographic) $ &\\
$promotes(environment, preservation) $ & $promotes(agriculture, irrigation) $\\
$HydrographicObject(river)$ & $HydrographicObject(lake)$\\
$FloraObject(wetland)$ & $FloraObject(grassland)$\\
$FloraObject(naturalForest)$ &\\
$over(preservation, wetland)$ & $over(irrigation, grassland)$\\
$associatedWith(wetland, lake)$ &\\
$associatedWith(naturalForest, river)$ &\\
$River(queguay)$ & $River(santaLucia)$\\
$Lake(deRocha)$ & $Lake(delSauce)$\\
$Wetland(staLuciaWetland)$ & $Wetland(deRochaWetland)$\\
$NaturalForest(castillosPalmForest)$ & $NaturalForest(queguayForest)$\\[0.5ex]
\hline
\end{tabular}
\end{tabular}
\end{center}

\begin{center}
\begin{tabular}{lr}
{\bf Mbox} \\
\begin{tabular}{|lll|}
\hline & & \\[-0.8ex]
$river \eqm  River$ & $wetland \eqm Wetland$ &$hydrographic \eqm HydrographicObject$  \\
$lake \eqm Lake$
 &  $naturalForest \eqm NaturalForest$ & $flora \eqm FloraObject$ \\[0.5ex]
 \hline
\end{tabular}
\end{tabular}
\end{center}

}
\caption{The $\mathcal{ALCQM}$-ontology of  Figure \ref{fig:secondView}}

\label{figure:boxes} 
\end{figure}

\section{Tableau Algorithm for $\mathcal{ALCQ}$}
\label{appendix:tableauALCQ}

We first   recall the tableau algorithm for checking consistency in
$\mathcal{ALCQ}$ \cite{journals/ai/BaaderBH96,FOST}. 
We follow the presentation of \cite{FOST} using tableau graphs
and make some small changes to be able to accommodate equalities and
inequalities in the Abox.
As in \cite{FOST}, we assume that 
all concepts in $\mathcal{O} = (\mathcal{T}, \mathcal{A})$
are in negation normal form and that $\mathcal{T}$ is a set of concepts.
Each  statement $C \sqsubseteq D$ of the Tbox is transformed into  
the concept $\neg C \sqcup D$.

\begin{definition}[Tableau Graph] 
A \emph{tableau graph}  consists of

\begin{itemize}

\item a set of nodes, labelled with  individual names
or variable names,

\item directed edges between some pairs of nodes,

\item for each node labelled $x$, $\lab(x)$
 is either undefined
or defined. If it is defined then it is a set of concept expressions, 
\item for each pair of nodes $x$ and $y$,  $\lab(x, y)$
is either undefined or defined. If it is defined, then it is a set 
of role names,
\item two relations between nodes, denoted by $\eql$ and $\neql$.
These relations keep record of the equalities and inequalities 
 of nodes  in the algorithm. 
 The relation $\eql$ is assumed to be
 reflexive, symmetric and transitive while $\neql$ is  assumed to be symmetric. 
 Canonical representatives are distinguished from non-canonical ones
 by setting  $\lab$ to  be defined or undefined.

\end{itemize}
\end{definition}

\begin{definition}[Equating two nodes]
\label{definition:equate}
We define a procedure 
$\equate{x}{y}{\lab}$ 
that equates two nodes $x$ and $y$ in $\lab$   as follows.
Let $x'$ and $y'$ be the canonical representatives of 
the $\eql$-equivalence classes of $x$ and $y$, i.e.
$x \eql x'$ and $y \eql y'$ where $\lab(x')$ and $\lab(y')$ are defined.
Assume that either $x'$ is not a variable or that
both $x'$ and $y'$ are variables 
(then, $x$ and $y$ will also be variables). 
When we equate a variable with an individual of
 the ontology, we choose the individual of the ontology as 
 representative of the equivalence class.
\begin{enumerate}

\item  set $\lab(x') \leftarrow \lab(y') $ 
\item  set  $\lab(x', z) \leftarrow \lab(y', z)$ and $\lab(z, x') \leftarrow \lab(z, y')$
\item  set $x' \eql y'$ 
\item set  $\lab(y') = \lab(y', z) = \lab(z,y')$ to be undefined 
\item for all $u$ with $u \neql y'$, $\makedifferent{u}{x'}{\lab}$
\item close $\eql$ under reflexivity, symmetry and transitivity.
\end{enumerate}
 
\end{definition}

\begin{definition}[Making two nodes different]
\label{definition:makedifferent}
We define a procedure called 
$\makedifferent{x}{y}{\lab}$ that makes two nodes $x$ and $y$ different in $\lab$   as follows.
For all $x'$ and $y'$ such that
$x \eql x'$ and $y \eql y'$, set $x' \neql y'$. 
Close $\neql$ under symmetry.
\end{definition}

\begin{definition}[Tableau Initialization]
\label{definition:tableauinit}
 The initial tableau for ${\cal O} = (\mathcal{T}, \mathcal{A})$
 is defined by the following procedure.

\begin{enumerate}
\item For each individual $a \in \mathcal{O}$, 
create a node labelled $a$ and  set $\lab(a) = \emptyset$.

\item For all pairs $a, b \in \mathcal{O}$ of individuals, set $\lab(a,b) = \emptyset$.

\item For each $C(a)$ in ${\cal A}$, set $\lab(a) \leftarrow C$.
  
\item For each $R(a,b)$ in ${\cal A}$, set $\lab(a,b) \leftarrow R$.

\item For each  $a \not = b$ in $\mathcal{A}$, set $a \neql b$.

\item For each $a = b$ in ${\cal A}$,
$\equate{a}{b}{\lab}$.
%
 
\end{enumerate}
\end{definition}

We say that  $y$ is  a {\em successor} of $x$
if  $\lab(x,y) $ is  neither $\emptyset$ nor undefined.
We define that $y$ is a {\em descendant} of $x$ by induction.
\begin{enumerate}
\item 
Every successor  of $x$,  which is a variable,
 is a descendant  of $x$. 
\item 
Every successor  of a descendant of $x$, which is  a variable, 
is also a descendant of $x$.
\end{enumerate}

\begin{definition}[Blocking]
We define the notion of blocking by induction.
A node $x$ is  blocked by a node $y$ if
$x$ is a descendant of $y$  and $\lab(x) \subseteq \lab(y)$
or $x$ is a descendant of $z$ and $z$ is blocked by $y$. 
\end{definition}

After initialization, the tableau algorithm proceeds by non-deterministically applying the {\bfseries{expansion rules}}  for $\mathcal{ALCQ}$ defined in Figure \ref{definition:rules}.

\begin{figure}
{\small
\begin{tabular}{ll}
{\bf $\sqcap$-rule}:& 
If $C \sqcap D \in \lab(x)$ and $\{C, D \} \not \subseteq \lab(x)$ then set $\lab(x) \leftarrow \{C, D \}$.\\
\\
{\bf $\sqcup$-rule}:& 
If $C \sqcup D \in \lab(x)$ and $\{C, D \} \cap \lab(x)= \emptyset$ then 
set $\lab(x) \leftarrow \{C \}$ or \\
                           & $\lab(x) \leftarrow \{D \}$.\\
\\
{\bf $\exists$-rule}:& 
If $x$ is not blocked, $\exists R.C\in \lab(x)$ and there is no $y$ with $R \in \lab(x,y)$ \\
                           & and $C \in \lab(y)$ then \\
& 1. Add a new node with label $y$ (where $y$ is a new node label), \\
& 2. set $\lab(x,y) = \{R\}$,\\
& 3. set $\lab(y) = \{ C \}$. \\
\\
{\bf $\forall$-rule}:& 
If $\forall R.C\in \lab(x)$ and there is a node $y$ with $R \in \lab(x,y)$ and $C \not \in \lab(y)$ \\
              & then set $\lab(x) \leftarrow C$. \\
\\
{\bf ${\cal T}$-rule}: & 
If $C \in {\cal T}$ and $C \not \in \lab(x)$, then $\lab(x) \leftarrow C$. \\

\\
{\bf $\geq$-rule}: & 
If $\geq nR.C \in \lab(x)$, $x$ is not blocked and there are no $y_1, \ldots, y_n$ such that\\
&  $R \in \lab(x, y_i)$, $C \in \lab(y_i)$, $y_i \neql y_j$ for $i,j \in \{1, \ldots n\}$, then  \\

& 1. create $n$ new nodes  $y_1, \ldots y_n$.\\
& 2. set $\lab(x,y_i) = \{R\}$, $\lab(y_i) = \{ C \}$ and $y_i \neql y_j$ for $i, j \in \{1, \ldots n \}$.\\

\\
{\bf choose-rule}: &
If $\leq n R. C \in \lab(x)$ and there is $y$ such that $R \in \lab(x,y)$,  
$C \not \in \lab(y)$, \\
&
$\fnn(\neg C) \not \in \lab(y)$, then set $\lab(y) \leftarrow C $ or $\lab(y) \leftarrow \fnn(\neg C) $. \\
\\
{\bf $\leq$-rule}: & 
If $\leq n R. C \in \lab(x)$, there are $y_1, \ldots, y_{n+1}$ with $R \in \lab(x, y_i)$, 
$C \in \lab(y_i)$  \\
& for $i \in \{1, \ldots n+1\}$ and there are $j,k \in \{1, \ldots n+1 \}$ such that  
$y_j \neql y_k$  \\ & does not hold, then 
 $ \equate{y_j}{y_k}{\lab}$
%
\end{tabular} 
}
\caption{Expansion Rules for $\mathcal{ALCQ}$}
\label{definition:rules}
\end{figure}

\begin{definition}[Contradiction]
$\mathcal{L}$  has a contradiction if either 

\begin{itemize}

\item $A$ and $\neg A$ belongs to $\mathcal{L}(x)$ for some atomic concept $A$ and node $x$ or 

\item  
 we have that $x \eql y$ and $x \neql y$ for some nodes $x$ and $y$.

\item there is a node $x$ such that $\leq nR.C \in \lab(x)$, $R \in \lab(x, y_i)$, $C \in \lab(y_i)$, $y_i  \neql y_j$ for all $i, j \in \{1, \ldots n +1\}$.

\end{itemize}

\end{definition}

\begin{definition}[${\cal ALCQ}$-Complete]
$\lab$ is ${\cal ALCQ}$-complete if 
none of the rules
of Figure \ref{definition:rules} is applicable.
\end{definition} 
The algorithm 
terminates when we reach some $\lab$ that is either complete
or has a contradiction. 
 The ontology $(\mathcal{T}, \mathcal{A})$ 
is consistent if there exists some
$\lab$ without contradictions.
Otherwise it is inconsistent.

\section{Omitted Proofs}
\label{appendix:proofs}

In this section, we show some  proofs that we could not
include in the main part of the paper due to  space constraints.\\\\
{\bf Proof of Lemma \ref{lemma:firstreasoningproblem}.}

\begin{proof}
First we prove the $\Rightarrow$ direction:\\
Suppose   towards a  
contradiction that there exists a model $\mathcal{I}$ of $\mathcal{O}$ such that $\mathcal{I} \models \mathcal{O} \cup \{a \not = b, b \eqm A \}$. Then, $a^{\mathcal{I}} \not = b^{\mathcal{I}}$ and $b^{\mathcal{I}} = A^{\mathcal{I}}$. But as $\mathcal{O} \models a \eqm A$, we have that $a^{\mathcal{I}} = A^{\mathcal{I}}$, $b^{\mathcal{I}} = A^{\mathcal{I}}$ and $a^{\mathcal{I}} \not = b^{\mathcal{I}}$, what results in a contradiction.
\\\\
$\Leftarrow$ direction:
\\
Suppose towards a 
 contradiction that $\mathcal{O} \not \models a \eqm A$. 
 Then for some model $\mathcal{I}$ of $\mathcal{O}$, $a^{\mathcal{I}} \not = A^{\mathcal{I}}$.
 We introduce a new individual $b$ such that $b^{\int}= A^{\int}$ and clearly,
 $b^{\int} \not = A^{\int}$. 
  This contradicts the hypothesis.
\end{proof}
{\bf Proof of Lemma \ref{lemma:secondreasoningproblem}.}

\begin{proof}
By Definition \ref{definition:metaconcept},  $C$ is a meta-concept iff 
$\mathcal{O} \models  C(a)$ and $\mathcal{O} \models a \eqm A$. 
 It is easy to see that $\mathcal{O} \models  C(a)$
is equivalent to 
the statement that $\mathcal{O} \cup \{\neg C(a) \}$  is  unsatisfiable.
\end{proof}

\begin{definition}[Abox associated to a Tableau graph]
\label{definition:associatedontology}
Given a graph $\lab$, the Abox 
 $\aboxL$  associated to $\lab$ is defined as follows.
\[
 \begin{array}{ll}
 \aboxL = & \{ C(x) \mid C \in \lab(y ), y \eql x\} \cup \\ 
          &      \{R(x,y) \mid R \in \lab(x',y'), x \eql x', y \eql y'  \} \cup  \\
          &      \{ x = y \mid x \eqL y  \} \cup
                \{ x \not = y \mid x \neql y \}
\end{array} 
 \]
\end{definition}

\begin{lemma}
\label{lemma:modelaboxinic}
Let $\int$ be a model of $(\mathcal{T}, \mathcal{A}, \mathcal{M})$ and $\lab$ the initial tableau for $(\mathcal{T}, \mathcal{A}, \mathcal{M})$. Then $\int$ is also a model of $(\mathcal{T}, \aboxL, \mathcal{M})$.
\end{lemma}


\begin{proof}
By Definitions \ref{definition:tableauinit} and 
 \ref{definition:associatedontology} 
 we have that:

\[\begin{array}{ll}
\aboxL = & 
\{ C(a)  \mid b \eqL a, C(b) \in \mathcal{A}\} \cup   \\
& \{R(a, b) \mid c \eqL a, d \eqL b, R(c, d) \in \mathcal{A}\} \cup  \\ 
&  \{a = b \mid a \eqL b\}\cup \{ a \not = b \mid  a \eqL c, b \eqL d, c \not = b \in 
\mathcal{A} \}  
\end{array}
\]  
It is also easy to see that
$\eqL$ is the reflexive, symmetric and transitive closure
of $\{(a,b) \mid a = b \in \mathcal{A}\}$.
Clearly, $\int$ is a model of $\mathcal{A}$ iff $\int$ is a model of $\aboxL$.
\end{proof}

The following lemma is easy to prove.

\begin{lemma}
\label{lemma:expansionrulesforward}
Let $\int$ be a model of $(\mathcal{T}, \aboxL, \mathcal{M})$.
\begin{enumerate}
\item 
If $(\mathcal{T'}, \lab')$  is 
 obtained from $(\mathcal{T}, \lab)$ by applying an
$\mathcal{ALCQM}$-expansion rule which is not 
close, $\leq$, $\sqcup$  or choose
then  $\int$ is a model
of $(\mathcal{T'}, \aboxLi{\lab'}, \mathcal{M})$.
\item 
If the rule applied is either 
close, $\leq$, $\sqcup$  or choose then
there exists a choice that yields a $(\mathcal{T'}, \lab')$
where $\int$ is a model
of $(\mathcal{T'}, \aboxLi{\lab'}, \mathcal{M})$.
\end{enumerate}
\end{lemma}
{\bf Proof of Soundness}.
 
\begin{proof}
By Lemma \ref{lemma:modelaboxinic}, we have that
$\int$ is a model of 
$(\mathcal{T}_0, \aboxLi{\lab_0}, \mathcal{M})$ 
where $\mathcal{T}_0 = \mathcal{T}$ 
and $\lab_0$ is the initial graph build by the tableau algorithm.
By Theorem \ref{theorem:termination}, the tableau algorithm 
always terminates.
It follows from Lemma \ref{lemma:expansionrulesforward}
and the fact that  $\int $ is a model  
of  $(\mathcal{T}, \mathcal{A}, \mathcal{M})$ 
that there is a sequence  $(\mathcal{T}_0, \lab_0), (\mathcal{T}_1, \lab_1),
 \ldots (\mathcal{T}_k,  \lab_k)$
 such that
 $\int$ is also a model of $(\mathcal{T}_k, \aboxLi{\lab_k}, \mathcal{M})$,
 $(\mathcal{T}_{i+1}, \lab_{i+1})$ 
is obtained from $(\mathcal{T}_i, \lab_i)$ by applying an 
$\mathcal{ALCQM}$-expansion rule,
 $(\mathcal{T}_k, \lab_k)$ is $\mathcal{ALCQM}$-complete
and  $\lab_k = \lab$ has no contradictions. \\
Suppose now towards a contradiction that $\lab$ has a cycle. 
By Definition \ref{definition:cycle},
there exist a set of meta-modelling axioms 
 $ A_0  \eqm a_0 $, $ A_1  \eqm a_1 $, $ \ldots$ $ A_n \eqm  a_n$
 all in   $\mathcal{M}$ 
such that 
\[\begin{array}{ll}
A_1 \in \lab(x_0) & x_0 \eql a_0\\
A_2 \in \lab(x_1) & x_1 \eql a_1\\
\vdots & \vdots \\
A_n \in \lab(x_{n-1}) \ \ \ & x_{n-1} \eql a_{n-1}\\
A_0 \in \lab(x_{n}) & x_{n} \eql a_{n} 
\end{array}
\]
It follows from Definition  \ref{definition:associatedontology}
 and  the fact that $\int$ is a model of  
 $(\mathcal{T}_k, \aboxLi{\lab_k})$ 
 $(\mathcal{T}_k, \lab_k)$
 ~\footnote{If $x \eqL a$ and $A \in \lab(x)$, 
 by Definition of $\aboxL$, $A (x) \in \aboxL$.
 Since $\int$ is a model of $\aboxL$, we have that
$(x)^{\int} = (a)^{\int} \in (A)^{\int}$
because $\int$ is a model of $\aboxL$. } 
that
\begin{equation}
\label{equation:completeness}
(a_n)^{\int} \in (A_0)^{\int}  \ \ (a_0)^{\int}  \in (A_1)^{\int}  \ \ (a_1)^{\int}  \in (A_2)^{\int}  
   \ldots
   (a_{n-1})^{\int} \in (A_{n})^{\int} 
\end{equation}
Since $\int$ is also a model of $\mathcal{M}$, we have that
\[
(a_n)^{\int} \in (A_0)^{\int} = (a_0)^{\int}  \in 
 (A_1)^{\int} = (a_1)^{\int}  \in (A_2)^{\int} 
   \ldots
  (a_{n-1})^{\int} \in (A_{n})^{\int} =  (a_n)^{\int}   
\]
and hence, the domain of $\int$ is not  well-founded
 contradicting the first clause in Definition \ref{definition:modelALCQM}. 
\end{proof}
{\bf Proof of Lemma \ref{lemma:modelstar}.}

\begin{proof}
Suppose $\lab$ is $\mathcal{ALCQM}$-complete and has no contradictions.
By the close-rule, we have that
either $a \eql  a'$ or $ a \not \eql a' $ (but not both).\\
Suppose 
 $a \eql a'$. By the $\eql$ and $\mathcal{T}$-rules,
  we have that
  $\{ A \sqcup \neg A'$, $A' \sqcup \neg A \}
\subseteq \lab(y)$ for all nodes $y$ such that $\lab(y)$ is defined.
It is  easy to prove that
$A \in \lab(y)$ iff $A' \in \lab(y)$ for all nodes $y$ such that $\lab(y)$
is defined.
Hence, $(A)^{\intcan} =(A')^{\intcan}$.\\ 
Suppose now that $a \neql a'$.  
By $\neql$-rule, we have that
$(A \sqcap \neg B \sqcup B \sqcap \neg A)(z)$
for some node $z$. Either $A, \neg B \in \lab(z)$ or
$B, \neg A \in \lab(z)$. In any case, $A^{\intcan} \not = B^{\intcan}$.
\end{proof}
\end{document}